\documentclass[12pt,a4paper]{article}
\usepackage[top=25mm, bottom=25mm, left=15mm, right=15mm]{geometry}

\usepackage{authblk}
\usepackage{amsmath,amssymb,amsthm}
\usepackage{comment}
\usepackage{cite}
\usepackage{multirow}
\usepackage{url}
\usepackage{tikz}
\usetikzlibrary{positioning,arrows.meta}
\usepackage{lineno}

\title{Singularities of Non-negative Matrix Factorization and their application to Bayesian inference}
\author{Naoki Hayashi\thanks{E-mail: \texttt{Naoki.Hayashi.wm@mosk.tytlabs.co.jp}}}
\author{Yota Maeda\thanks{E-mail: \texttt{yota.maeda@mosk.tytlabs.co.jp}}}
\author{Yasushi Esaki\thanks{E-mail: \texttt{Yasushi.Esaki.sb@mosk.tytlabs.co.jp}}}
\affil{Toyota Central R\&D Labs., Inc.}
\date{2026/9/28}

\newtheorem{thm}{Theorem}[section]
\newtheorem{defi}{Definition}[section]
\newtheorem{lem}{Lemma}[section]
\newtheorem{conj}{Conjecture}[section]
\newtheorem{cor}{Corollary}[section]

\newtheorem{prop}{Proposition}[section]

\newcommand{\blockA}[1]{%
\begin{bmatrix}
#1 & O \\
O & O
\end{bmatrix}
}

\begin{document}
\maketitle

\begin{abstract}
Non-negative matrix factorization (NMF) is a singular statistical
model whose Bayesian asymptotics are governed by the real log
canonical threshold (RLCT).
We study the local geometry of the factorization map and derive
an upper bound for the RLCT of NMF.
Let $H$ be the model inner dimension and $H_0$ the non-negative
rank of the true $M\times N$ matrix.
Assuming that the true matrix admits a strictly positive
factorization of inner dimension $H_0$ in the interior of the
parameter domain, we prove, for smooth positive priors, that
$\lambda\leq
\{(H-H_0)\min(M,N)+H_0(M+N-H_0)\}/2$.
This bound strictly improves the previous bound when $H_0\geq3$.
The proof uses a local analytic normal form that separates
independent linear coordinates from a residual matrix product.
When $H=H_0$ also equals the ordinary rank of the true matrix,
we obtain the exact value $\lambda=H_0(M+N-H_0)/2$.
Under the standard assumptions of singular learning theory,
these results bound the leading coefficients of the expected
Bayesian generalization error and the Bayesian free energy.
\end{abstract}


\section{\label{sec_intro}Introduction}
A matrix whose entries are non-negative is called a non-negative matrix.
Non-negative matrix factorization (NMF) is a method for representing a non-negative matrix as the product of two non-negative matrices~\cite{Paatero, Lee}.
Representing a non-negative data matrix in this manner enables applications to knowledge discovery in the real world, and NMF has indeed been used in image recognition~\cite{Lee}, text mining~\cite{Xu}, signal processing~\cite{Virtanen2008bayesian}, bioinformatics~\cite{Kim}, consumer analysis~\cite{Kohjima}, and recommender systems~\cite{Bobadilla2018recommender}.
On the other hand, since there are infinitely many pairs of matrices whose products represent the same matrix, NMF based on objective-function optimization depends strongly on the initial values and often converges to a local solution.
One method for addressing this problem is Bayesian NMF~\cite{Cemgil}.
This method uses Gibbs sampling to compute the posterior distribution of the matrix parameters and is more robust than conventional iterative optimization methods because, among other advantages, it can assess fluctuations in the estimation results.
Because a Gibbs sampler can be derived, variational inference based on a mean-field approximation can also be formulated~\cite{Cemgil}.

A statistical model in which the map from the parameters to probability distributions is injective and the Fisher information matrix is positive definite is called a regular model.
In contrast, in statistical models such as neural networks, reduced-rank regression, normal mixtures, latent Dirichlet allocation, and Markov models, the map from the parameters to probability distributions cannot be injective because these models have hierarchical structures or hidden variables.
Such models are called singular models.
Most statistical models used in machine learning are singular models~\cite{Watanabe2007almost}, and NMF described above is also a singular model.
Their likelihoods and posterior distributions cannot be approximated by normal distributions~\cite{SWatanabeBookE, SWatanabeBookMath},
and Bayesian inference is known to be more effective for singular models than maximum likelihood estimation or maximum a posteriori estimation in the sense that it can yield a smaller generalization loss~\cite{SWatanabeBookE, SWatanabeBookMath}.
Determining the theoretical value of the Bayesian generalization loss is important for designing models and hyperparameters and estimating the required sample size, and several models have been studied to date, for example, neural networks whose activations are analytic-odd~\cite{Watanabe2} and linear~\cite{Aoyagi1}, mixture models such as normal~\cite{Yamazaki1} and Bernoulli~\cite{Yamazaki2013comparing}, and latent variable models including latent Dirichlet allocation~\cite{nhayashi9} and Markov models~\cite{Zwiernik2011asymptotic}.

As mentioned above, Bayesian inference methods for NMF have been proposed.
Their Bayesian generalization errors and variational approximation errors have also been studied~\cite{nhayashi2, nhayashi5, nhayashi8}.
These previous studies analyze the real log canonical threshold (RLCT), which describes the leading term in the asymptotic behavior of the expected Bayesian generalization error.
However, a tighter evaluation remains conjectural and involves the non-negative rank~\cite{Cohen}, the minimum inner dimension of a matrix product satisfying the non-negativity constraints~\cite{nhayashi5}.
In this study, we derive a tighter upper bound for the RLCT of NMF than previous bounds by analyzing singularities.
This tighter upper bound not only provides a more accurate characterization of the asymptotic behavior of the Bayesian generalization error but also improves implementations of RLCT-based model selection methods~\cite{Drton} and posterior computation methods~\cite{Nagata2008asymptotic} by providing a closer substitute for the exact RLCT.

The rest of the paper is organized as follows.
Section~\ref{sec_bayesframe} introduces the Bayesian inference framework.
Section~\ref{sec_main} states the Main Theorem, which is proved in Section~\ref{sec_proof}.
Section~\ref{sec_discuss} discusses the Main Theorem, and Section~\ref{sec_conc} concludes the paper.

%
%

\section{\label{sec_bayesframe}Bayesian inference framework}
Suppose that a dataset of $n$ independent random variables $D_n=(X_1, \ldots, X_n)$ is given, and consider performing Bayesian inference using a statistical model $p(x|w)$.
Here, for each $i=1,\ldots,n$, $X_i$ is assumed to follow the true distribution $q(x)$.
Let $\mathcal{W} \subset \mathbb{R}^d$ be the parameter set and let $w \in \mathcal{W}$.
Also, let $\varphi(w)$ denote the prior distribution.
The marginal likelihood and the posterior distribution are denoted by $Z_n$ and $\varphi^*(w|D_n)$, respectively.
The definition of conditional probability gives
\begin{align}
\varphi^*(w|D_n) &= \frac{\prod_{i=1}^n p(X_i|w) \varphi(w)}{Z_n}, \\
Z_n &= \int \prod_{i=1}^n p(X_i|w) \varphi(w) dw.
\end{align}
Define the probability distribution $p^*(x|D_n)$ as the posterior expectation of the statistical model:
\begin{align}
p^*(x|D_n) = \int p(x|w)\varphi^*(w|D_n) dw.
\end{align}
This distribution is called the Bayesian predictive distribution.

Bayesian inference is performed from a given pair $(p,\varphi)$ and dataset $D_n$;
however, it is necessary to evaluate whether $(p,\varphi)$ is appropriate for the true distribution $q(x)$.
To assess the accuracy of the inference, typical evaluation criteria are the free energy and the generalization error~\cite{SWatanabeBookMath, StatRethinkMcElreath2nd}.
The free energy is defined in terms of the negative log marginal likelihood by
\begin{align}
F_n = -\log Z_n
\end{align}
and the Bayesian generalization error is defined as
\begin{align}
G_n &=\int q(x)\log \frac{q(x)}{p^*(x|D_n)}dx.
\end{align}
The free energy and Bayesian generalization error, respectively, represent the accuracy of inference for the data-generating process through the marginal likelihood and the accuracy of prediction for unseen data through the Bayesian predictive distribution.
It has been proved that these quantities have the following asymptotic expansions:
\begin{align}
F_n &= nS_n + \lambda \log n + O_p(\log \log n), \label{eq_asympF} \\
\mathbb{E}[G_n] &= \frac{\lambda}{n}  + o\left( \frac{1}{n} \right), \label{eq_asympG}
\end{align}
where $\mathbb{E}[\cdot]$ denotes the expectation with respect to the true distribution of the dataset $D_n$ and $S_n$ is the empirical entropy of $q(x)$, 
and $\lambda$ is the RLCT determined by $(q,p,\varphi)$~\cite{SWatanabeBookE, SWatanabeBookMath}.
Thus, the RLCT dominates the free energy and the generalization error through the coefficient of the leading terms.
Clarifying the theoretical free energy and the generalization error is effective for tuning hyperparameters and estimating sufficient sample size.
Besides, model selection methods that directly use the theoretical value of the RLCT~\cite{Drton} and a design method for the exchange Monte Carlo method~\cite{Nagata2008asymptotic} have been proposed; thus, determining RLCTs is important not only from theoretical but also engineering perspectives.


\section{\label{sec_main}Main Theorem}

We state the theorem proved in this study.
In this paper, we suppose the following assumptions to develop the determination of RLCT~\cite{SWatanabeBookE, SWatanabeBookMath, Yamazaki1, Zwiernik2011asymptotic, nhayashi9}.
\begin{itemize}
    \item (A1) The parameter set is a semianalytic and compact subset in Euclidean space, and its interior is not empty.
    \item (A2) The prior $\varphi(w)$ can be simultaneously resolved with $K(w)=\mathrm{KL}(q \mid p)$ and is locally integrable.
    \item (A3) There exists a non-negative analytic function $H(w)$ such that $K(w)$ is equivalent to  $H(w)$. 
\end{itemize}
Note that the equivalence between $K(w)$ and $H(w)$ is defined by the existence of positive constants $c_1$ and $c_2$ such that $c_1 H(w) \leqq K(w) \leqq c_2H(w)$.
The above assumptions have been implicitly assumed in previous research to find RLCTs of singular models.
Thus, based on the assumptions in our previous studies~\cite{nhayashi2, nhayashi5, nhayashi8}, we define the relevant concepts as follows.

Let $\mathrm{M}(M,N,\mathcal{C})$ denote the set of $M \times N$ real matrices whose entries belong to $\mathcal{C} \subset \mathbb{R}$, and let $\mathrm{Hom}(M,N)$ denote the set of linear maps from $\mathbb{R}^M$ to $\mathbb{R}^N$.
In particular, we abbreviate $\mathrm{M}(M,N,\mathbb{R})$ as $\mathrm{M}(M,N)$.
Let $\mathrm{GL}(r)$ be the set of invertible real matrices of degree $r$, which we identify with the general linear group of degree $r$.
$\mathcal{K}$ denotes a compact and semianalytic subset of the non-negative real numbers that contains 0 and whose interior is not empty.
Let $\mathcal{K}_0$ be a compact subset of the positive real numbers.
Set $\Theta = \mathrm{M}(M,H,\mathcal{K}) \times \mathrm{M}(H,N,\mathcal{K})$.
Take matrices $U \in \mathrm{M}(M,H,\mathcal{K})$ and $V \in \mathrm{M}(H,N,\mathcal{K})$, and write them as
\begin{align}
U&=(u_1,\ldots,u_H)=(u_{ik})_{i=1,k=1}^{M,H}, \\
V&=(v_1,\ldots,v_H)^\top=(v_{kj})_{k=1,j=1}^{H,N}.
\end{align}

Furthermore, take matrices $U_0 \in \mathrm{M}(M,H_0,\mathcal{K}_0)$ and $V_0 \in \mathrm{M}(H_0,N,\mathcal{K}_0)$ such that the non-negative rank of $U_0V_0$ is $H_0$.
It satisfies $H_0 \leqq \min\{ M,N\}$.
We write them as
\begin{align}
U_0&=(u^0_1,\ldots,u^0_{H_0})=(u^0_{ik})_{i=1,k=1}^{M,H_0}, \\
V_0&=(v^0_1,\ldots,v^0_{H_0})^\top=(v^0_{kj})_{k=1,j=1}^{H_0,N}.
\end{align}
We also denote the $M \times N$ zero matrix by $O_{M,N}$ and abbreviate it as $O$ when its size is clear, for example, within a block matrix.

Let $H \geqq H_0$,
define the map $F: \Theta \to \mathrm{M}(M,N)$ by $F(U,V)=UV-U_0V_0$, and
define the non-negative analytic function $\Phi: \Theta \to \mathbb{R}$ by $\Phi(U,V)=\lVert UV-U_0V_0 \rVert^2$~\footnote{NMF models having the same RLCT as this function include those based on the normal, Poisson, and exponential distributions. It can be proved that these correspond to NMF with the squared error~\cite{Paatero, Lee}, I-divergence~\cite{Finesso}, and Itakura--Saito divergence~\cite{Itakura}, respectively, as their loss functions~\cite{nhayashi2}.},
where the matrix norm is the Frobenius norm.
We also assume that $\Theta_0=\Phi^{-1}(0)\cap \Theta$ is nonempty.

\begin{defi}[RLCT of NMF~\cite{nhayashi2, nhayashi5, nhayashi8}]
\label{defi_nmf_rlct}
Let $\varphi(U,V)$ be a smooth prior distribution.
Under appropriate assumptions,
the univariate complex function
\begin{align}
\zeta(z)=\iint_\Theta dUdV \Phi(U,V)^z \varphi(U,V)
\end{align}
is holomorphic for $\mathrm{Re}(z)>0$ and admits a unique analytic continuation to the entire complex plane as a meromorphic function, all of whose poles are negative rational numbers{\rm~\cite{SWatanabeBookE, SWatanabeBookMath}}.
When the largest pole is $(-\lambda)$,
$\lambda$ is called the RLCT of {\rm NMF}.
\end{defi}

The Main Theorem of this study is as follows.

\begin{thm}[Main Theorem]\label{thm_main}
If the prior distribution is smooth, positive and bounded on $\Phi^{-1}(0)$,
then the RLCT $\lambda$ of {\rm NMF} satisfies the following inequality:
\begin{align}
\lambda \leqq \frac{1}{2}\left\{
(H-H_0)\min\{M,N\} +H_0(M+N-H_0)
\right\}.
\end{align}
If $H=H_0=r$, then
\[
  \lambda=\frac{H_0(M+N-H_0)}{2}.
\]
\end{thm}

The proof is given in the next section.
By the Main Theorem, for cases in which $\mathrm{KL}(q \mid p)$ is equivalent to $\Phi(U,V)$, that is, there exist constants $c_1>0$ and $c_2>0$ such that $c_1 \Phi(U,V) \leqq \mathrm{KL}(q \mid p) \leqq c_2 \Phi(U,V)$, the Bayesian generalization error and Bayesian free energy can be evaluated as follows.

\begin{cor}[Bayesian generalization error and Bayesian free energy for NMF]\label{cor_nmf}
Suppose that $\mathrm{KL}(q \mid p)$ is equivalent to $\Phi(U,V)$ and the same assumptions as Theorem{\rm~\ref{thm_main}} hold.
The Bayesian generalization error $G_n$ and Bayesian free energy $F_n$ satisfy the following inequalities:
\begin{align}
\mathbb{E}[G_n] &\leqq \frac{1}{2n}\left\{
(H-H_0)\min\{M,N\} +H_0(M+N-H_0)
\right\} + o\left( \frac{1}{n} \right), \\
F_n &\leqq nS_n + \frac{1}{2}\left\{
(H-H_0)\min\{M,N\} +H_0(M+N-H_0)
\right\} \log n \nonumber \\
&\quad + O_p\left( \log \log n \right),
\end{align}
where $S_n$ is the empirical entropy of the true distribution $q(x)$.
\end{cor}

Although NMF sometimes considers the factorization of only one non-negative matrix, it generally deals with multiple data matrices.
This is because, as in purchase analysis, matrix data may be obtained on a daily or monthly basis or from different locations~\cite{Takeuchi2013NM2F, Kohjima}.
In such cases, the factorization of multiple matrices can be reduced to statistical inference, making it possible to statistically evaluate the generalization error and free energy of Bayesian inference as above.

In Bayesian inference for NMF, a Poisson--gamma model~\cite{Cemgil} is often used. In this case, however, the prior distribution may have zeros and hence does not satisfy the assumption of Theorem \ref{thm_main}.
Even in this case, the main result of this study can be used to derive a tighter upper bound for the RLCT than those in previous studies.
Define the gamma distributions used as the prior distribution~$\varphi_{\mathrm{Gam}}(U,V)$ for the matrix parameters $U$ and $V$ by
\begin{align}
\mathrm{Gam}(U|\phi_U, \theta_U) &= \prod_{i=1}^M \prod_{k=1}^H \frac{\theta_U^{\phi_U}}{\Gamma(\phi_U)}u_{ik}^{\phi_U-1}e^{-\theta_U u_{ik}}, \\
\mathrm{Gam}(V|\phi_V, \theta_V) &= \prod_{k=1}^H \prod_{j=1}^N \frac{\theta_V^{\phi_V}}{\Gamma(\phi_V)}v_{kj}^{\phi_V-1}e^{-\theta_V v_{kj}}, \\
\varphi_{\mathrm{Gam}}(U,V)&=\mathrm{Gam}(U|\phi_U, \theta_U)\mathrm{Gam}(V|\phi_V, \theta_V),
\end{align}
where the hyperparameters $\phi_U,\theta_U,\phi_V,\theta_V$ are all positive real numbers.
The domain of this prior can be restricted to a compact set $\Theta$ that includes $\Theta_0$~\cite{nhayashi8}.

\begin{cor}[Corollary of the Main Theorem]\label{cor_hyper}
If the prior is $\varphi(U,V)=\varphi_{\mathrm{Gam}}(U,V)$, then the RLCT $\lambda$ of {\rm NMF} satisfies the following inequality:
\begin{align}
\lambda \leqq \frac{1}{2}\left\{
(H-H_0)\min\{M\phi_U,N\phi_V\} +H_0(M+N-H_0)
\right\}.
\end{align}
\end{cor}

We state the position of the Main Theorem in the body of knowledge.
There are a few studies related to the RLCT of NMF.
First, for matrix factorization without non-negativity constraints, the exact RLCT has been obtained in all cases~\cite{Aoyagi1}.
Its proof uses the fact that the true rank is $r$ to transform the matrices into a canonical form containing the $r \times r$ identity matrix, and then partitions the product of the parameter matrices into an $r \times r$ part and the remaining parts using block matrices.
In contrast, for NMF, which imposes non-negativity constraints, one must consider the non-negative rank, a quantity distinct from the usual rank.
The non-negative rank is generally greater than or equal to the usual rank, and examples of non-negative matrices whose two types of rank strictly differ are known in~\cite{Cohen}.
Therefore, it is non-trivial whether a method similar to that of~\cite{Aoyagi1} is effective in determining the RLCT of NMF.

Second, in conventional research, upper bounds for the RLCT of NMF have thus far been derived by other methods.
The exact RLCT has been obtained when the true parameter is a zero matrix and when the inner dimensions of the matrix products for the statistical model and the true distribution are equal to one or two; conventional upper bounds have been derived by using these results~\cite{nhayashi2, nhayashi5}.
For variational inference in NMF, the asymptotic behavior of the variational free energy has also been clarified~\cite{Kohjima2017phase}.
In contrast, a lower bound for the variational approximation error has been obtained by deriving an upper bound for the RLCT when the same prior distribution and statistical model are used,
and it has also been shown that Bayesian inference and variational inference have different phase-transition structures with respect to the hyperparameters of the prior distribution~\cite{nhayashi8}.

The Main Theorem improves the the upper bound for the RLCT of NMF obtained in previous studies.
In Section~\ref{sec_discuss}, we show that the derived upper bound is strictly tighter than the conventional ones, and the Main Theorem partially resolves the existing conjecture~\cite{nhayashi5} for the exact value of the RLCT.
Moreover, a tighter lower bound of the variational approximation error can be immediately obtained by applying Corollary~\ref{cor_hyper} and Kohjima's result~\cite{Kohjima2017phase} to the previous proof~\cite{nhayashi8}.

\section{\label{sec_proof}Proof of the Main Theorem}

First, we recall a known result concerning the RLCT of unconstrained matrix factorization.
\begin{defi}[RLCT of Matrix Factorization~\cite{Aoyagi1}]\label{def_mf_rlct}
Let $A \in \mathrm{M}(I,K)$, $B \in \mathrm{M}(K,J)$, $A_0 \in \mathrm{M}(I,K_0)$, and $B_0 \in \mathrm{M}(K_0,J)$,
where $I \geqq 1$, $J \geqq 1$, $K \geqq K_0 \geqq 0$, and $K_0=\mathrm{rank}(A_0B_0)$.
Suppose that $(A,B)$ ranges over a compact subset $\Omega$ of the $K(I+J)$-dimensional Euclidean space and that $(A,B)$ satisfying $AB=A_0B_0$ is contained in $\Omega$.
Under suitable assumptions, the complex function of one variable
\begin{align}
\zeta_{\mathrm{MF}}(z)=\iint dAdB \lVert AB-A_0B_0 \rVert^{2z}
\end{align}
is holomorphic in the region $\mathrm{Re}(z)>0$, but it admits a unique analytic continuation to the entire complex plane as a meromorphic function, and all its poles are negative rational numbers{\rm~\cite{SWatanabeBookE, SWatanabeBookMath}}.
When its largest pole is $(-\lambda_{\mathrm{MF}})$,
$\lambda_{\mathrm{MF}}$ is called the RLCT of matrix factorization.
\end{defi}

\begin{thm}[Aoyagi and Watanabe~\cite{Aoyagi1}]\label{thm_aoyagi}
The RLCT of matrix factorization $\lambda_{\mathrm{MF}}=\lambda_{\mathrm{MF}}(I,J,K,K_0)$ is given as follows.
\begin{enumerate}
\item If $I+K_0 \leqq J+K$, $J+K_0 \leqq I+K$, and $K+K_0 \leqq I+J$,
\begin{enumerate}
\item if $I+J+K+K_0$ is even,
$$\lambda_{\mathrm{MF}}(I,J,K,K_0)=\frac{1}{8}\{2(K+K_0)(I+J)-(I-J)^2-(K+K_0)^2\}.$$
\item if $I+J+K+K_0$ is odd,
$$\lambda_{\mathrm{MF}}(I,J,K,K_0)=\frac{1}{8}\{2(K+K_0)(I+J)-(I-J)^2-(K+K_0)^2+1\}.$$
\end{enumerate}
\item If $I+K_0 > J+K$,
$$\lambda_{\mathrm{MF}}(I,J,K,K_0)=\frac{1}{2}(KJ-KK_0+IK_0).$$
\item If $J+K_0 > I+K$,
$$\lambda_{\mathrm{MF}}(I,J,K,K_0)=\frac{1}{2}(KI-KK_0+JK_0).$$
\item If $K+K_0 > I+J$,
$$\lambda_{\mathrm{MF}}(I,J,K,K_0)=\frac{1}{2}IJ.$$
\end{enumerate}
\end{thm}

Theorem \ref{thm_aoyagi} will be used to prove a lemma below.
We next state the lemmas used in the proof of the Main Theorem.
To show the statements, we use the following notation.
Let $W_0=U_0V_0$ and the matrix $F(U,V)=UV-W_0$ as $F=(F_{ij})_{i=1,j=1}^{M,N}$.
We assume that the chosen factorization belongs to the interior of
$\Theta$.
Also, let $\mathcal{Z}$ be the zero set of $F$:
$\mathcal{Z}=\{(U,V)\in\Theta \mid UV=W_0\}$.
In the proof of the Main Theorem, we first consider the case $H=H_0$ in detail and then use the result to derive an upper bound for general $H \geqq H_0$.
Thus, when $H=H_0$, we choose the true factorization $(U_0,V_0)$ of $W_0$ as a point at which the RLCT attains its minimum on $\mathcal{Z}$.
Let the ranks at this point be
$a=\mathrm{rank}(U_0)$, $b=\mathrm{rank}(V_0)$, and $r=\mathrm{rank}(W_0)=\mathrm{rank}(U_0V_0)$
and put $R=MN-(M-a)(N-b)$ and $c=H_0-a-b+r$.

\begin{lem}[Lemma 3.1 in~\cite{nhayashi8}]\label{lem_null}
Consider the case $\Phi(U,V)=\lVert UV \rVert^2$.
If the prior distribution $\varphi(U,V)$ is positive and bounded on
$\Phi^{-1}(0)$, then its RLCT $\lambda_{\mathrm{null}}$ is
\begin{align}
\lambda_{\mathrm{null}} = \frac{H \min \{M,N \}}{2}.
\end{align}
Moreover, when the prior satisfies
$\varphi(U,V)=\varphi_{\mathrm{Gam}}(U,V)$,
\begin{align}
\lambda_{\mathrm{null}} = \frac{H \min \{M\phi_U, N\phi_V \}}{2}.
\end{align}
\end{lem}

\begin{lem}\label{lem_linmap_rank}
Consider the linear map determined by matrices $A \in \mathrm{M}(M,H)$ and $B \in \mathrm{M}(H,N)$,
\begin{align}
\mathcal{L}_{(A,B)} &: \mathrm{M}(M,H) \times \mathrm{M}(H,N) \to \mathrm{M}(M, N), \qquad (X,Y) \mapsto XB+AY
\end{align}
If $a=\mathrm{rank}(A)$ and $b=\mathrm{rank}(B)$, then
\begin{align}
\mathrm{rank}(\mathcal{L}_{(A,B)})
=MN-(M-a)(N-b)
=Mb+aN-ab. \label{eq_rank_linear_map}
\end{align}
\end{lem}

\begin{lem}\label{lem_ideal_equal}
Suppose $H=H_0$.
If $M>a$, $N>b$, and $c>0$, then there exist analytic local coordinates $(z,X,Y,\xi)$ in a neighborhood of $(U_0,V_0)$,
such that
$z\in\mathbb{R}^{R}$, $X\in\mathrm{M}(M-a,c)$, $Y\in\mathrm{M}(c,N-b)$
and
\begin{align}
\left\langle (F_{ij})_{i=1,j=1}^{M,N} \right\rangle
=
\left\langle
z_1,\ldots,z_R,
((XY)_{\mu\nu})_{\mu=1,\nu=1}^{M-a,N-b}
\right\rangle
\label{eq_local_ideal}
\end{align}
hold, where $(XY)_{\mu\nu}$ is the $(\mu,\nu)$ entry of the matrix $XY \in \mathrm{M}(M-a,N-b)$ and $\langle \cdot \rangle$ is a generated ideal.

If $M=a$, $N=b$, or $c=0$, then there exist analytic local coordinates $(z,\xi)$ in a neighborhood of $(U_0,V_0)$ such that $z \in \mathbb{R}^R$ and
\begin{align}
\left\langle (F_{ij})_{i=1,j=1}^{M,N} \right\rangle
=
\left\langle
z_1,\ldots,z_R
\right\rangle
\label{eq_local_ideal_size_is_rank}
\end{align}
hold.
\end{lem}

\begin{lem}\label{lem_local_reg}
Let $H=H_0$.
Under the same assumptions as in Lemma$~\ref{lem_ideal_equal}$,
denote the RLCT obtained by restriction to a sufficiently small neighborhood of $(U_0,V_0)$ by
$\lambda_{(U_0,V_0)}$.
Furthermore,
if the prior distribution is positive and bounded in this neighborhood, then
\begin{align}
\lambda
=\lambda_{(U_0,V_0)}
=\frac{R}{2}+\tilde{\lambda}_{\mathrm{MF}}(M-a,N-b,c,0)
\leqq \frac{H_0(M+N-H_0)}{2}
\end{align}
holds, where $\tilde{\lambda}_{\mathrm{MF}}(M-a,N-b,c,0)$ is $\lambda_{\mathrm{MF}}$ in Theorem$~\ref{thm_aoyagi}$ when $M>a$, $N>b$, and $c>0$, and is $0$ otherwise.
In particular, if $H_0=r$, then
\begin{align}
\lambda=\frac{H_0(M+N-H_0)}{2}
\end{align}
holds.
\end{lem}

An outline of the proof of the Main Theorem using these lemmas is shown below.
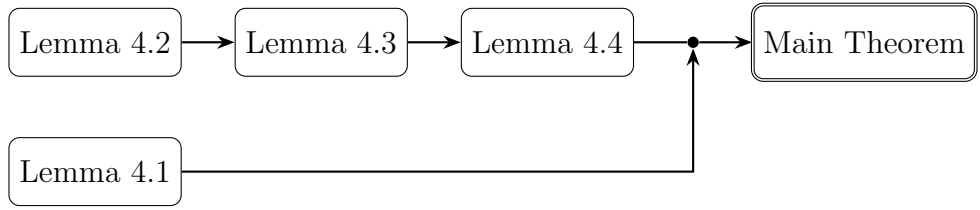
\begin{center}
\begin{tikzpicture}[
    lemma/.style={
        draw,
        rounded corners,
        minimum width=22mm,
        minimum height=9mm,
        align=center
    },
    theorem/.style={
        draw,
        double,
        rounded corners,
        minimum width=25mm,
        minimum height=10mm,
        align=center
    },
    arrow/.style={
        -{Stealth},
        thick
    },
    node distance=12mm and 18mm
]

\node[lemma] (B) {Lemma \ref{lem_linmap_rank}};
\node[lemma, right=7mm of B] (C) {Lemma \ref{lem_ideal_equal}};
\node[lemma, right=7mm of C] (D) {Lemma \ref{lem_local_reg}};
\node[lemma, below=8mm of B] (A) {Lemma \ref{lem_null}};

\node[
    circle,
    fill,
    inner sep=1.5pt,
    right=7mm of D
] (join) {};

\node[theorem, right=7mm of join] (main) {Main Theorem};

\draw[arrow] (B) -- (C);
\draw[arrow] (C) -- (D);

\draw[thick] (D) -- (join);
\draw[arrow] (A) -| (join);
\draw[arrow] (join) -- (main);

\end{tikzpicture}
\end{center}

Lemma~\ref{lem_null} has been proved in previous studies~\cite{nhayashi2,nhayashi8}.
We prove Lemmas~\ref{lem_linmap_rank}, \ref{lem_ideal_equal}, and \ref{lem_local_reg} below.

\begin{proof}[{\bf Proof of Lemma~\ref{lem_linmap_rank}}]
Choose nonsingular matrices $P \in \mathrm{GL}(M)$, $Q \in \mathrm{GL}(H)$, $R \in \mathrm{GL}(H)$, and $S \in \mathrm{GL}(N)$ such that
\begin{align}
PAQ=\blockA{I_a}, \quad
RBS=\blockA{I_b}.
\end{align}
Then, for $(X,Y)$,
\begin{align}
P(\mathcal{L}_{(A,B)}(X,Y))S &= PXBS + PAYS \\
&= PXR^{-1}RBS + PAQQ^{-1}YS \\
&=PXR^{-1}\blockA{I_b}
+\blockA{I_a}Q^{-1}YS.
\end{align}
Setting $\tilde{X}=PXR^{-1}$ and $\tilde{Y}=Q^{-1}YS$, the coordinate transformation $(X,Y)\mapsto(\tilde{X},\tilde{Y})$ is invertible,
and the transformation $Z\mapsto PZS$ on the output side is also invertible.
Therefore, the rank of $\mathcal{L}_{(A,B)}$ is equal to that of
\begin{align}
(\tilde{X},\tilde{Y})
\mapsto
\tilde{X}\blockA{I_b}+\blockA{I_a}\tilde{Y}.
\end{align}
In this canonical form, $\tilde{X}\blockA{I_b}$ allows only the first $b$ columns to vary freely,
while $\blockA{I_a}\tilde{Y}$ allows only the first $a$ rows to vary freely.
Therefore, the image consists of all $M\times N$ matrices whose lower-right $(M-a)\times(N-b)$ block is zero.
The dimension of the image is
\begin{align}
\mathrm{rank}(\mathcal{L}_{(A,B)}) = MN-(M-a)(N-b) = Mb+aN-ab.
\end{align}
\end{proof}

\begin{proof}[{\bf Proof of Lemma~\ref{lem_ideal_equal}}]
By assumption,
\begin{align}
0 \leqq r \leqq \min\{a,b\}, \quad r \leqq H_0, \quad \max\{a,b\} \leqq H_0.
\end{align}
Regarding $V_0$ and $U_0$ as linear maps from $\mathbb{R}^N$ to $\mathbb{R}^{H_0}$ and
from $\mathbb{R}^{H_0}$ to $\mathbb{R}^M$, respectively, the image of the restricted map $U_0|_{\mathrm{Im}(V_0)}:\mathrm{Im}(V_0)\to\mathbb{R}^M$ is $\mathrm{Im}(U_0V_0)$.
Thus, by the rank--nullity theorem,
\begin{align}
r=\mathrm{rank}(U_0V_0)
=\dim(\mathrm{Im}(V_0))-\dim(\mathrm{Ker}(U_0|_{\mathrm{Im}(V_0)}))
=b-\dim(\mathrm{Im}(V_0)\cap\mathrm{Ker}(U_0)).
\end{align}
Therefore,
$\dim(\mathrm{Im}(V_0)\cap\mathrm{Ker}(U_0))=b-r \geqq 0$.
Moreover, since $\dim(\mathrm{Ker}(U_0))=H_0-a$, we have $\dim(\mathrm{Im}(V_0)\cap\mathrm{Ker}(U_0)) \leqq H_0-a$, and hence
\begin{align}
r &= b-\dim(\mathrm{Im}(V_0)\cap\mathrm{Ker}(U_0)) \\
&\geqq b+a-H_0.
\end{align}
Setting $c=H_0-a-b+r$, we have $c \geqq 0$ and
\begin{align}
c=\dim(\mathrm{Ker}(U_0))-\dim(\mathrm{Im}(V_0)\cap\mathrm{Ker}(U_0)).
\end{align}

We now choose a direct-sum decomposition of the intermediate space $\mathbb{R}^{H_0}$,
\begin{align}
\mathbb{R}^{H_0}=E_r\oplus E_t\oplus E_s\oplus E_c
\end{align}
such that
\begin{align}
\dim(E_r)=r,\quad \dim(E_t)=b-r,\quad
\dim(E_s)=a-r,\quad \dim(E_c)=c,
\end{align}
and
\begin{align}
E_r\oplus E_t=\mathrm{Im}(V_0),\qquad
E_t\oplus E_c=\mathrm{Ker}(U_0).
\end{align}
Since invertible linear transformations of the parameter and output spaces do not change the RLCT, we can assume that
\begin{align}
U_0=
\begin{bmatrix}
I_r & O & O & O \\
O & O & I_{a-r} & O \\
O & O & O & O
\end{bmatrix}, \quad
V_0=
\begin{bmatrix}
I_r & O & O \\
O & I_{b-r} & O \\
O & O & O \\
O & O & O
\end{bmatrix} \label{eq_simultaneous_standard_form}
\end{align}
without loss of generality.
The row-block sizes of $U_0$ are $r,a-r,M-a$, the column-block sizes of $V_0$ are $r,b-r,N-b$, respectively, and they are ordered as $E_r,E_t,E_s,E_c$.

We set $U=U_0+\Delta U$ and $V=V_0+\Delta V$ and obtain
\begin{align}
F(U,V)=UV-W_0=\Delta U V_0+U_0\Delta V+\Delta U\Delta V.
\end{align}
Define its linear part by
\begin{align}
L(\Delta U,\Delta V)=\Delta U V_0+U_0\Delta V.
\end{align}
Then, by Lemma~\ref{lem_linmap_rank}, we have
\begin{align}
\mathrm{rank}(L)=R=MN-(M-a)(N-b). \label{eq_rank_L_corrected}
\end{align}
The standard form Eq.~\eqref{eq_simultaneous_standard_form} shows that the image of $L$ is precisely the set of all matrices whose lower-right $(M-a)\times(N-b)$ block is zero.

Partitioning the output into $a$ and $M-a$ rows and into $b$ and $N-b$ columns, write
\begin{align}
U=\begin{bmatrix}U_A\\ U_B\end{bmatrix},\qquad
V=\begin{bmatrix}V_L&V_R\end{bmatrix}
\end{align}
where
\begin{align}
&U_A\in\mathrm{M}(a,H_0),\quad
U_B\in\mathrm{M}(M-a,H_0),\\
&V_L\in\mathrm{M}(H_0,b),\quad
V_R\in\mathrm{M}(H_0,N-b).
\end{align}
Let $W_{11}$ denote the upper-left $a\times b$ block of $W_0$. By the standard form Eq.~\eqref{eq_simultaneous_standard_form}, we may write
\begin{align}
W_0=
\begin{bmatrix}
W_{11}&O\\
O&O
\end{bmatrix},
\qquad
W_{11}\in\mathrm{M}(a,b),\quad
\mathrm{rank}(W_{11})=r.
\end{align}
Consequently, we have
\begin{align}
F(U,V)=UV-W_0=
\begin{bmatrix}
U_AV_L-W_{11}&U_AV_R\\
U_BV_L&U_BV_R
\end{bmatrix}.
\label{eq_local_blocks}
\end{align}

Let $F$ be the matrix on the right-hand side
and its block matrices are denoted by
$\tilde F_{11}=U_AV_L-W_{11}$, $\tilde F_{12}=U_A V_R$, $\tilde F_{21}=U_B V_L$, and $\tilde F_{22}=U_B V_R$, respectively.
Collect the three blocks other than the lower-right block as
\begin{align}
z=(\tilde F_{11}, \tilde F_{12}, \tilde F_{21})
=(U_AV_L-W_{11},\,U_AV_R,\,U_BV_L).
\label{eq_local_z}
\end{align}
The number of components of $z$ is
$ab+a(N-b)+(M-a)b=R$ and Eq.~\eqref{eq_rank_L_corrected} implies that $dz_{(U_0,V_0)}$ is surjective.
Thus, by the analytic submersion theorem, $z$ can be taken as part of a system of analytic local coordinates in a neighborhood of $(U_0,V_0)$.
Denoting the remaining coordinates by $\eta$, we obtain the analytic submanifold
\begin{align}
S=\{(z,\eta) \mid z=0\},
\label{eq_local_S}
\end{align}
where the origin of the coordinates $(z,\eta)$ is $(U_0,V_0)$.
On $S$, we have
\begin{align}
U_AV_L=W_{11},\qquad U_AV_R=O,\qquad U_BV_L=O.
\end{align}
By taking the neighborhood sufficiently small, we may ensure that
$\mathrm{rank}(U_A)=a$ and $\mathrm{rank}(V_L)=b$ throughout the neighborhood.
Thus, $U_A$ is surjective and $V_L$ is injective.
On $S$, we set
linear spaces
$K=\mathrm{Ker}(U_A)$ and $J=K\cap\mathrm{Im}(V_L)$.
These are vector spaces satisfying $J \subset K$. The equality $U_AV_L=W_{11}$ gives 
\begin{align}
J=V_L(\mathrm{Ker}(W_{11})). \label{eq_essential_space_in_V}
\end{align}
Moreover, since $\mathrm{rank}(W_{11})=r$, the injectivity of $V_L$ yields
\begin{align}
\dim J=b-r,\qquad
\dim K=H_0-a,\qquad
\dim(K/J)=H_0-a-b+r=c,
\label{eq_local_quotient}
\end{align}
where $K/J$ is the quotient vector space obtained by identifying vectors in $K$ that differ by an element of $J$.
Note that the above quotient space can be $\{0 \}$.
If $H_0=r$, then it immediately follows that $a=b=r$ and $c=0$; in this case, $K=\{ 0 \}$, $J=\{ 0 \}$, and $K/J=\{ \{ 0 \} \} \cong \{ 0 \}$.
Let $K_0$ and $J_0$ be the value of $K$ and $J$ at $(U_0, V_0)$, respectively.

Since the rank of $U_A$ remains $a$ in the neighborhood,
there exists a regular $a \times a$ submatrix of $U_A$.
Using this submatrix, choose a right inverse that depends analytically on $U_A$ and satisfies
\begin{align}
R_A:\mathbb{R}^a\longrightarrow\mathbb{R}^{H_0},
\qquad U_AR_A=I_a.
\end{align}
Put $P_K=I_{H_0}-R_AU_A$, we have $U_AP_K=O$ and $P_Kv=v$ for any $v\in K$.
Hence, $P_K$ is an analytic projection onto $K$.
Let $E_c^0\in\mathrm{M}(H_0,c)$ be the matrix whose column vectors form a basis of $E_c$, and set $E=P_KE_c^0$.
At $(U_0,V_0)$, we have $E=E_c^0$ and
$K_0=J_0\oplus E_c$.
Therefore, after shrinking the neighborhood if necessary, we have
\begin{align}
K=J\oplus\mathrm{Im}(E).
\label{eq_local_direct_sum}
\end{align}

Let $T\in\mathrm{M}(b,b-r)$ be the matrix whose column vectors form a basis of $\mathrm{Ker}(W_{11})$. Then, by Eq.~\eqref{eq_essential_space_in_V}, the columns of $V_LT$ form a basis of $J$, and
\begin{align}
B=\begin{bmatrix}V_LT&E\end{bmatrix}
\in\mathrm{M}(H_0,H_0-a).
\end{align}
Then the columns of $B$ form a basis of $K$ by Eq.~\eqref{eq_local_direct_sum}.
Since some square submatrix of $B$ of order $H_0-a$ is nonsingular throughout the neighborhood, choose a left inverse that depends analytically on $B$ in the form
\begin{align}
L_B=\begin{bmatrix}\rho\\ \pi\end{bmatrix},
\qquad L_BB=I_{H_0-a},
\end{align}
where $\rho$ and $\pi$ have $b-r$ and $c$ rows, respectively. Then
\begin{align}
\pi V_LT=O,\qquad \pi E=I_c
\label{eq_local_pi}
\end{align}
and any $v\in K$ can be uniquely expressed as
\begin{align}
v=V_LT\,\rho v+E\,\pi v.
\end{align}
Thus, the map $K/J\to\mathbb{R}^c$ induced by $\pi$ is an analytic isomorphism.

Since $U_AV_R=O$ on $S$, we have $\mathrm{Im}(V_R)\subset K$.
Applying the above representation to each column of $V_R$ gives
\begin{align}
V_R=V_LT\,\rho V_R+E\,\pi V_R.
\end{align}
On the other hand, $U_BV_L=O$ implies $U_BV_LT=O$, and hence
\begin{align}
U_BV_R=U_BE\,\pi V_R.
\end{align}
Set $X=U_BE\in\mathrm{M}(M-a,c)$ and $Y=\pi V_R\in\mathrm{M}(c,N-b)$.
Then, on $S$,
\begin{align}
F_{22}=U_BV_R=XY
\label{eq_local_XY}
\end{align}
holds as an identity of analytic functions.

We show that $X$ and $Y$ are independent local coordinates on $S$.
At $(U_0,V_0)$, we have $U_{B,0}=O$, $V_{R,0}=O$, and $E_0=E_c^0$, so
\begin{align}
dX_0(\Delta U,\Delta V)&=\Delta U_BE_c^0,\\
dY_0(\Delta U,\Delta V)&=\pi_0\Delta V_R.
\end{align}
If only the columns of $\Delta U_B$ corresponding to $E_c$ are varied, then
$\Delta U_BV_{L,0}=O$, so this variation belongs to $T_0S$.
This variation produces an arbitrary $\Delta X\in\mathrm{M}(M-a,c)$.
Similarly, if only the rows of $\Delta V_R$ corresponding to $E_c$ are varied, then
$U_{A,0}\Delta V_R=O$, and an arbitrary $\Delta Y\in\mathrm{M}(c,N-b)$ is obtained.
It follows that
\begin{align}
d(X,Y)_0:T_0S\longrightarrow
\mathrm{M}(M-a,c)\times\mathrm{M}(c,N-b)
\end{align}
is surjective.
The analytic submersion theorem now allows us to take $(X,Y,\xi)$ as analytic local coordinates on $S$. Furthermore, by Eq.~\eqref{eq_local_S}, analytically extending these coordinates to the ambient neighborhood makes
\begin{align}
(z,X,Y,\xi)
\label{eq_local_coordinates}
\end{align}
a system of analytic local coordinates in a neighborhood of $(U_0,V_0)$.

By Eq.~\eqref{eq_local_XY},
\begin{align}
\tilde F_{22}(0,X,Y,\xi)=XY.
\end{align}
Thus, for any $\mu,\nu$,
$F_{22,\mu\nu}-(XY)_{\mu\nu}$ vanishes identically on $z=0$.
By Hadamard's lemma for analytic functions, there exist analytic functions $H_{i,\mu\nu}$ such that
\begin{align}
\tilde F_{22,\mu\nu}
=(XY)_{\mu\nu}
+\sum_{i=1}^Rz_iH_{i,\mu\nu}(z,X,Y,\xi).
\label{eq_local_hadamard}
\end{align}
The components of $F$ outside the lower-right block are precisely $z_1,\ldots,z_R$.
Therefore, denoting the ideal generated by analytic functions $f_1,\ldots,f_D$ by $\langle f_1, \ldots, f_D \rangle$, we have
\begin{align}
\left\langle (F_{ij})_{i=1,j=1}^{M,N} \right\rangle
=\left\langle z_1,\ldots,z_R, (\tilde F_{22,\mu\nu})_{\mu=1,\nu=1}^{M-a,N-b} \right\rangle
=\left\langle z_1,\ldots,z_R,((XY)_{\mu\nu})_{\mu=1,\nu=1}^{M-a,N-b} \right\rangle.
\end{align}
Moreover, if $a=M$ or $b=N$, then $F_{22}$ is clearly empty.
If $c=0$, then $XY=O$ and hence $\tilde{F}_{22}$ vanishes identically on $S$.
Thus, in either case, we have
\begin{align}
\left\langle (F_{ij})_{i=1,j=1}^{M,N} \right\rangle
=\left\langle z_1,\ldots,z_R \right\rangle.
\end{align}
\end{proof}

\begin{proof}[{\bf Proof of Lemma~\ref{lem_local_reg}}]
The RLCT is the minimum of the local RLCTs over all zeros.
Let $(U_0,V_0)$ be a zero attaining this minimum, and consider $\Phi(U,V)=\lVert F \rVert^2$ in a neighborhood of this point.
By Lemma~\ref{lem_ideal_equal}, there exist analytic local coordinates $(z,X,Y,\xi)$ such that $z \in \mathbb{R}^R$, $X \in \mathrm{M}(M-a,c)$, $Y \in \mathrm{M}(c, N-b)$ and
\begin{align}
\left\langle (F_{ij})_{i=1,j=1}^{M,N} \right\rangle
=\left\langle z_1,\ldots,z_R,(XY)_{\mu\nu}\right\rangle.
\end{align}
When $M=a$, $N=b$, or $c=0$, we write $XY=O$.

Owing to~\cite{SWatanabeBookE}, it follows that $\Phi(U,V)$ has the same RLCT $\lambda_{(U_0,V_0)}$ as
\begin{align}
\lVert z \rVert^2 + \lVert XY \rVert^2 = \sum_{k=1}^R z_k^2 + \sum_{\mu=1}^{M-a} \sum_{\nu=1}^{N-b} (XY)_{\mu\nu}^2.
\end{align}
Furthermore, since $z$ and $XY$ are independent and RLCTs are additive with respect to independent variables~\cite{SWatanabeBookE}, we obtain
\begin{align}
\lambda_{(U_0, V_0)}&=\lambda_{\mathrm{vec}}+\tilde{\lambda}_{\mathrm{MF}} \label{eq_local_reg_exact_rlct}, \\
\lambda_{\mathrm{vec}} &= \frac{R}{2}=\frac{1}{2}\{MN-(M-a)(N-b)\} \label{eq_local_reg_smoothpart_rlct}, \\
\tilde{\lambda}_{\mathrm{MF}} &= \begin{cases}
0 & M=a \textrm{ or } N=b \textrm{ or } c=0, \\
\lambda_{\mathrm{MF}}(M-a, N-b, c, 0) & M>a \textrm{ and } N>b \textrm{ and } c>0,
\end{cases} \label{eq_local_reg_matpart_rlct}
\end{align}
where $\lambda_{\mathrm{MF}}$ is the RLCT of $\lVert XY \rVert^2$ when $M>a$, $N>b$, and $c>0$.
We now determine $\lambda_{(U_0,V_0)}$.

Case (a): Suppose that $M>a$, $N>b$, and $c>0$.
First, since $c=H_0-a-b+r$, we always have
\begin{align}
M+N-a-b-2c = (M-H_0) + (N-H_0) + (a-r) + (b-r) \geqq 0,
\end{align}
that is,
\begin{align}
c \leqq 2c \leqq M+N-a-b. \label{eq_main_not_reach_aoyagi_case4}
\end{align}
Thus, $\lambda_{\mathrm{MF}}(M-a,N-b,c,0)$ is obtained by substituting $I \leftarrow M-a$, $J \leftarrow N-b$, $K \leftarrow c$, and $K_0 \leftarrow 0$ into Theorem~\ref{thm_aoyagi}.
Together with Eqs.~\eqref{eq_local_reg_exact_rlct}--\eqref{eq_local_reg_matpart_rlct}, this yields the following.
\begin{enumerate}
\item If $M-a \leqq N-b+c$ and $N-b \leqq M-a+c$, then
\begin{align}
&\quad I+J-c=M+N-a-b-c, \\
&\quad 2c(I+J)-(I-J)^2-c^2 = 4IJ - (I+J-c)^2
\end{align}
and hence
\begin{align}
&\quad \frac{1}{2}\{MN-(M-a)(N-b)\} +\frac{1}{8}\{2c(M+N-a-b)-(M-N-a+b)^2-c^2\} \\
&= \frac{1}{8}\{4MN-4(M-a)(N-b)+2c(M+N-a-b) - (M-N-a+b)^2-c^2\} \\
&= \frac{1}{8}\{4MN - 4IJ + 2c(I+J)-(I-J)^2-c^2\} \\
&= \frac{1}{8}\{4MN - 4IJ + 4IJ - (I+J-c)^2\} \\
&= \frac{1}{8}\{4MN - (M+N-a-b-c)^2\}.
\end{align}
Therefore, we obtain the following results.
\begin{enumerate}
\item If $M+N-a-b+c$ is even, then
\begin{align}
\lambda_{(U_0,V_0)}&=\frac{1}{8}\{4MN - (M+N-a-b-c)^2\},
\end{align}
\item If $M+N-a-b+c$ is odd, then
\begin{align}
\lambda_{(U_0,V_0)}=\frac{1}{8}\{4MN - (M+N-a-b-c)^2+1\}.
\end{align}
\end{enumerate}
\item If $M-a > N-b+c$, then
\begin{align}
\lambda_{(U_0,V_0)}=\frac{1}{2}\{MN-(N-b)(M-a-c)\}.
\end{align}
\item If $N-b > M-a+c$, then
\begin{align}
\lambda_{(U_0,V_0)}=\frac{1}{2}\{MN-(M-a)(N-b-c)\}.
\end{align}
\item The case $c>M+N-a-b$ cannot occur by Eq.~\eqref{eq_main_not_reach_aoyagi_case4}.
\end{enumerate}

Case (b): Suppose that $M=a$, $N=b$, or $c=0$.
The conditions and Eqs.~\eqref{eq_local_reg_exact_rlct}--\eqref{eq_local_reg_matpart_rlct} give
\begin{align}
\lambda_{(U_0,V_0)}=\frac{R}{2}=\frac{MN-(M-a)(N-b)}{2}.
\end{align}

By the minimality of the RLCT at $(U_0,V_0)$ and $H=H_0$, we have $\lambda_{(U_0,V_0)}=\lambda$.
We show that $H_0(M+N-H_0)/2$ is an upper bound in all cases.
First, consider Case (a).
When $M-a \leqq N-b+c$ and $N-b \leqq M-a+c$, define $\delta=[M+N-a-b-c\textrm{ is odd}]$ using the Iverson bracket.
By the definition of $c$, we have $M+N-a-b-c=M+N-H_0-r$. Setting
\begin{align}
M_0 = M-H_0, \quad N_0=N-H_0, \quad H_1=H_0-r
\end{align}
we have $M_0 \geqq 0$, $N_0 \geqq 0$, and $H_1 \geqq 0$, and
\begin{align}
&\quad (M+N-H_0-r)^2 \nonumber \\
&= (M_0+N_0+H_1)^2 \\
&= H_1^2 +2(M_0+N_0)H_1 +(M_0+N_0)^2 \\
&= H_1^2 +2(M_0-N_0)H_1 +(M_0-N_0)^2 + 4N_0H_1 + 4M_0N_0 \\
&= (H_1+M_0-N_0)^2+ 4N_0H_1 + 4M_0N_0. \label{eq_square_comp_for_H1}
\end{align}
Transposing $4M_0N_0$ in Eq.~\eqref{eq_square_comp_for_H1} and subtracting $\delta$ from both sides gives
\begin{align}
(M_0+N_0+H_1)^2 - 4M_0N_0 -\delta
= (H_1+M_0-N_0)^2 + 4N_0H_1 -\delta. \label{eq_eval_square_term_in_rlct}
\end{align}
Since $M_0+N_0+H_1-(H_1+M_0-N_0)=2N_0$,
$M_0+N_0+H_1$ and $H_1+M_0-N_0$ have the same parity.
Thus, when $\delta=1$, $H_1+M_0-N_0$ is also odd, so its square is at least 1.
Therefore, $(H_1+M_0-N_0)^2-\delta \geqq 0$.
Moreover, since clearly $4N_0H_1 \geqq 0$, Eq.~\eqref{eq_eval_square_term_in_rlct} implies
\begin{align}
&\quad (M_0+N_0+H_1)^2 - 4M_0N_0 -\delta \geqq 0 \\
&\Leftrightarrow (M+N-H_0-r)^2 - 4(M-H_0)(N-H_0) -\delta \geqq 0 \\
&\Leftrightarrow (M+N-H_0-r)^2 -\delta \geqq 4(M-H_0)(N-H_0).
\end{align}
These equivalences give
\begin{align}
4MN-(M+N-a-b-c)^2+\delta &\leqq 4MN-4(M-H_0)(N-H_0) \\
&= 4H_0(M+N-H_0)
\end{align}
and hence
\begin{align}
\lambda \leqq \frac{1}{8}\{4H_0(M+N-H_0)\} = \frac{H_0(M+N-H_0)}{2}.
\end{align}
For the remaining cases, including Case (b),
$a \leqq a+c \leqq H_0$ and $b \leqq b+c \leqq H_0$ imply
\begin{align}
\begin{cases}
MN-(N-b)(M-a-c) \leqq MN-(N-H_0)(M-H_0) = H_0(M+N-H_0), \\
MN-(M-a)(N-b-c) \leqq MN-(N-H_0)(M-H_0) = H_0(M+N-H_0), \\
MN-(M-a)(N-b) \leqq MN-(N-H_0)(M-H_0) = H_0(M+N-H_0).
\end{cases}
\end{align}
Thus, in all cases,
\begin{align}
\lambda \leqq \frac{H_0(M+N-H_0)}{2}.
\end{align}

Finally, suppose that $H_0=r$.
Since $H=H_0$, for any $(U,V)\in\mathcal{Z}$, we have
\begin{align}
H_0=\mathrm{rank}(W_0)=\mathrm{rank}(UV)
\leqq \mathrm{rank}(U)\leqq H_0, \\
H_0=\mathrm{rank}(W_0)=\mathrm{rank}(UV)
\leqq \mathrm{rank}(V)\leqq H_0
\end{align}
and therefore
\begin{align}
\mathrm{rank}(U)=\mathrm{rank}(V)=H_0.
\end{align}
Thus, at each zero,
\begin{align}
a=b=r=H_0, \quad c=H_0-a-b+r=0.
\end{align}
Using this fact and the minimality of the RLCT at $(U_0,V_0)$, Eqs.~\eqref{eq_local_reg_exact_rlct}--\eqref{eq_local_reg_matpart_rlct} yield
\begin{align}
\lambda = \frac{R}{2} = \frac{MN-(M-H_0)(N-H_0)}{2} = \frac{H_0(M+N-H_0)}{2}.
\end{align}
\end{proof}

We now prove the Main Theorem using the preceding lemmas.

\begin{proof}[{\bf Proof of the Main Theorem}]
First, as in previous studies~\cite{nhayashi2,nhayashi8},
we bound $\Phi(U,V)$ from above using the triangle inequality:
\begin{align}
\Phi(U,V)
&= \lVert UV - U_0V_0 \rVert^2 \\
&= \sum_{i=1}^M \sum_{j=1}^N
\left(
u_{i1}v_{1j} + \cdots + u_{iH}v_{Hj}
- u^0_{i1}v^0_{1j} - \cdots - u^0_{iH_0}v^0_{H_0j}
\right)^2 \\
&= \sum_{i=1}^M \sum_{j=1}^N
\left\{
\sum_{k=1}^{H_0} (u_{ik}v_{kj} - u^0_{ik}v^0_{kj})
+ \sum_{k=H_0+1}^H u_{ik}v_{kj}
\right\}^2 \\
&\leqq 2 \sum_{i=1}^M \sum_{j=1}^N
\left\{
\sum_{k=1}^{H_0} (u_{ik}v_{kj} - u^0_{ik}v^0_{kj})
\right\}^2 + 2 \sum_{i=1}^M \sum_{j=1}^N
\left(
\sum_{k=H_0+1}^H u_{ik}v_{kj}
\right)^2 .
\end{align}

For this upper bound, set
\begin{align}
\overline{\Phi}_1(U_{:,1:H_0},V_{1:H_0,:}) &= \sum_{i=1}^M \sum_{j=1}^N
\left\{
\sum_{k=1}^{H_0} (u_{ik}v_{kj} - u^0_{ik}v^0_{kj})
\right\}^2, \\
\overline{\Phi}_2(U_{:,(H_0+1):H},V_{(H_0+1):H,:}) &= \sum_{i=1}^M \sum_{j=1}^N
\left(
\sum_{k=H_0+1}^H u_{ik}v_{kj}
\right)^2, \\
\overline{\Phi}(U,V) &= \overline{\Phi}_1(U_{:,1:H_0}, V_{1:H_0,:}) + \overline{\Phi}_2(U_{:,(H_0+1):H}, V_{(H_0+1):H,:}).
\end{align}
Since multiplication by a positive constant does not change the RLCT~\cite{SWatanabeBookE}, and the RLCT of an upper bound of a function is an upper bound on the RLCT of the original function~\cite{SWatanabeBookE}, the RLCT $\overline{\lambda}$ of $\overline{\Phi}(U,V)$ gives an upper bound on the RLCT $\lambda$ of NMF.
We define the submatrices of $U,V$ as follows:
$U_{:,1:H_0} \in \mathrm{M}(M,H_0,\mathcal{K})$, $U_{:,(H_0+1):H} \in \mathrm{M}(M, H-H_0, \mathcal{K})$, $V_{1:H_0,:} \in \mathrm{M}(H_0, N, \mathcal{K})$, and $V_{(H_0+1):H,:} \in \mathrm{M}(H-H_0, N, \mathcal{K})$ so that
\begin{align}
U&=\overbrace{\begin{bmatrix}
U_{:,1:H_0} & U_{:,(H_0+1):H}
\end{bmatrix}
}^{H_0 \text{ and } H-H_0 \text{ columns}}, \\
V&=\left.\begin{bmatrix}
V_{1:H_0,:} \\
V_{(H_0+1):H,:}
\end{bmatrix}
\right\} H_0 \text{ and } H-H_0 \text{ rows}.
\end{align}
Let $\overline{\lambda}_i$ be the RLCT of $\overline{\Phi}_i$ $(i=1,2)$.
By the definitions of $\overline{\Phi}$, $\overline{\Phi}_i$, and the submatrices,
\begin{align}
\overline{\lambda} = \overline{\lambda}_1+\overline{\lambda}_2
\end{align}
holds~\cite{SWatanabeBookE}.
It therefore suffices to evaluate each term.

For the first term, we have
$\overline{\Phi}_1(U_{:,1:H_0},V_{1:H_0,:})=\lVert U_{:,1:H_0}V_{1:H_0,:}-U_0V_0 \rVert^2$; hence, by Lemma~\ref{lem_local_reg},
\begin{align}
\overline{\lambda}_1 \leqq \frac{H_0(M+N-H_0)}{2}.
\end{align}
On the other hand, for the second term, we have
$\overline{\Phi}_2(U_{:,(H_0+1):H}, V_{(H_0+1):H,:})=\lVert U_{:,(H_0+1):H}V_{(H_0+1):H,:}\rVert^2$; hence, by Lemma~\ref{lem_null},
\begin{align}
\overline{\lambda}_2 = \frac{(H-H_0)\min\{M,N\}}{2}.
\end{align}
Therefore,
\begin{align}
\lambda \leqq \frac{1}{2}\left\{
(H-H_0)\min\{M,N\} +H_0(M+N-H_0)
\right\}.
\end{align}
\end{proof}

The corollaries of the Main Theorem are proved as follows.

\begin{proof}[{\bf Proof of Corollary~\ref{cor_nmf}}]
The result follows by applying the Main Theorem to Eqs.~\eqref{eq_asympF} and \eqref{eq_asympG}.
\end{proof}

\begin{proof}[{\bf Proof of Corollary~\ref{cor_hyper}}]
The result follows by using the case $\varphi(U,V)=\varphi_{\mathrm{Gam}}(U,V)$ of Lemma~\ref{lem_null} for $\overline{\lambda}_2$ in the proof of the Main Theorem.
\end{proof}

\section{Discussion}\label{sec_discuss}
\subsection{Tightness of the derived bound}
First, to examine the degree of improvement in the upper bound, we compare the upper bound $\overline{\lambda}_{\mathrm{prev}}$ obtained in the previous study~\cite{nhayashi5} with the upper bound $\overline{\lambda}$ in the Main Theorem.
According to~\cite{nhayashi5}, if $\delta_{H_0}$ is defined as 1 when $H_0$ is odd and 0 when it is even, then
\begin{align}
\overline{\lambda}_{\mathrm{prev}} = \frac{1}{2}\left\{ (H-H_0)\min\{M,N\} + H_0(M+N-2)+\delta_{H_0} \right\}.
\end{align}
Therefore,
\begin{align}
\overline{\lambda}_{\mathrm{prev}} - \overline{\lambda}
&= \frac{1}{2}\left\{ (H-H_0)\min\{M,N\} + H_0(M+N-2)+\delta_{H_0} \right\} \notag \\
&\quad - \frac{1}{2}\left\{(H-H_0)\min\{M,N\} +H_0(M+N-H_0)\right\} \\
&= \frac{1}{2}\left( H_0^2 -2H_0 + \delta_{H_0} \right) \\
&=
\begin{cases}
\frac{1}{2} (H_0-1)^2 & \text{$H_0$ is odd}, \\
\frac{1}{2} \{(H_0-1)^2 -1\} & \text{$H_0$ is even}.
\end{cases}
\end{align}
When $H_0$ is odd, $\overline{\lambda}_{\mathrm{prev}} \geqq \overline{\lambda}$ clearly holds.
When $H_0$ is even, $\overline{\lambda}_{\mathrm{prev}} \geqq \overline{\lambda}$ also holds if $H_0 \geqq 2$ or $H_0=0$.
Thus, the Main Theorem gives a tighter upper bound than that in the previous study.

\subsection{Rank assumptions on the true parameter matrices}
Following previous studies, the Main Theorem considers the non-negative rank of $U_0V_0$ as information about the true distribution. However, if the ranks of $U_0$, $V_0$, and $U_0V_0$ are all known,
a tighter upper bound can be obtained using Lemma~\ref{lem_local_reg}.
Specifically, for $a=\mathrm{rank}(U_0)$, $b=\mathrm{rank}(V_0)$, and $r=\mathrm{rank}(U_0V_0)$,
setting $M_1=M-a$, $N_1=N-b$, and $c=H_0-a-b+r$ allows $\overline{\lambda}_1$ in the proof of the Main Theorem to be written as
\begin{align}
\overline{\lambda}_1 = \frac{MN-M_1 N_1}{2} + \tilde{\lambda}_{\mathrm{MF}}(M_1,N_1,c,0)
\end{align} 

and hence the RLCT of NMF is bounded as
\begin{align}
\lambda \leqq \frac{1}{2}\left\{ (H-H_0)\min\{M,N\} + MN-M_1 N_1 \right\}+\tilde{\lambda}_{\mathrm{MF}}(M_1,N_1,c,0),
\end{align}
where $\tilde{\lambda}_{\mathrm{MF}}$ is the quantity given by Eq.~\eqref{eq_local_reg_matpart_rlct}.
This inequality gives an even tighter upper bound than the Main Theorem, and equality holds when $H=H_0$ by Lemma~\ref{lem_local_reg}.
That is, specifying the ranks of $U_0$ and $V_0$ yields a tighter upper bound on the RLCT.
In practice, however, model selection for NMF often uses the inner dimension $H$ of the product of the parameter matrices as a control variable, in which case the true value of $H$ is $H_0$.
Indeed, because the singular Bayesian information criterion~\cite{Drton} assumes such a setting, the assumptions of the Main Theorem are natural: they assume the non-negative rank of $U_0V_0$ without specifying the individual ranks of $U_0$ and $V_0$.

\subsection{Relation to the Conjecture on the RLCT of NMF in Previous Work}
Previous work~\cite{nhayashi5} proposed the following conjecture concerning the RLCT $\lambda$ of NMF.
\begin{conj}[Conjecture V.1 in~\cite{nhayashi5}]
Let the non-negative rank and rank of $U_0V_0$ be $H_0$ and $r$, respectively.
\begin{enumerate}
\item If $H_0=r$, the RLCT of {\rm NMF} is equal to that of reduced-rank regression. In particular, if $H=H_0=r$, then $\lambda=r(M+N-r)/2$.
\item If $H_0>r$, the RLCT of {\rm NMF} is strictly greater than that of reduced-rank regression.
\end{enumerate}
\end{conj}
Note that the RLCT of reduced-rank regression is equal to that of matrix factorization in Theorem~\ref{thm_aoyagi}~\cite{Aoyagi1}.
This study partially resolves this conjecture.
Specifically, it establishes the conjectured equality when $H=H_0=r$ and shows that the equality can also hold in some cases with $H=H_0>r$.

By Lemma~\ref{lem_local_reg},
when $H=H_0=r$, we indeed have $\lambda=\lambda_{\mathrm{MF}}(M,N,r,r)$.
On the other hand, when $H=H_0>r$, $\lambda$ is not strictly greater than $\lambda_{\mathrm{MF}}$ in some cases.
We state the claim depending on $r=\mathrm{rank}(U_0V_0)$, $a=\mathrm{rank}(U_0)$ and $b=\mathrm{rank}(V_0)$, as follows.

\begin{prop}\label{prop_nmf_vs_mf_ranksmall}
If $H=H_0>r$, $\lambda \geqq \lambda_{\mathrm{MF}}$ holds.
The following conditions are necessary and sufficient for equality.
\begin{enumerate}
\item If $M-N>H-r$, then
\begin{equation}
  \lambda = \lambda_{\mathrm{MF}}
  \quad\Longleftrightarrow\quad b=r.
\end{equation}
\item If $N-M>H-r$, then
\begin{equation}
  \lambda = \lambda_{\mathrm{MF}}
  \quad\Longleftrightarrow\quad a=r.
\end{equation}
\item If $|M-N|\leqq H-r$, then
\begin{equation}
\lambda = \lambda_{\mathrm{MF}}
\quad\Longleftrightarrow\quad
\begin{cases}
a-r\displaystyle\leq
 \left\lceil\dfrac{M-N+H-r}{2}\right\rceil,\\[8pt]
b-r\displaystyle\leq
 \left\lceil\dfrac{N-M+H-r}{2}\right\rceil,
\end{cases}
\label{eq_balanced-equality}
\end{equation}
\end{enumerate}
where $\lceil \cdot \rceil$ is the ceiling function.
Whenever the applicable equality condition fails, the inequality is strict.
\end{prop}

To prove Proposition~\ref{prop_nmf_vs_mf_ranksmall},
we set the preliminaries as follows. 

Introduce the reduced variables:
\begin{align}
M_1 = M-r, \qquad N_1=N-r, \qquad H_1=H-r, \qquad
s=a-r, \qquad t=b-r.
\end{align}
For non-negative integers $I$, $J$, $K$ satisfying $K \leqq \min\{I,J\}$, define
\begin{align}
\ell(I,J,K)=\lambda_{\mathrm{MF}}(I,J,K,0).
\end{align}
Using Theorem~\ref{thm_aoyagi}, we have
\begin{align}
Q(I,J,K) &= 2K(I+J)-(I-J)^2-K^2, \\
\ell(I,J,K) &=
\begin{cases}
    \frac{1}{2}JK & I>J+K, \\
    \frac{1}{2}IK & J>I+K, \\
    \frac{1}{8}\left\{ Q(I,J,K)+\delta_{I,J,K} \right\} & |I-J|\leqq K,
\end{cases}\label{eq_zero-rank-aw}
\end{align}
where
\begin{align}
\delta_{I,J,K} = \begin{cases}
    0,&I+J+K\text{ is even},\\
    1,&I+J+K\text{ is odd}.
  \end{cases}
\end{align}

We use the following lemma to prove Proposition~\ref{prop_nmf_vs_mf_ranksmall}.
\begin{lem}[Rank-removal Identity]\label{lem_rank_removal_identity}
If $H+r \leqq M+N$,
the RLCT of matrix factorization satisfies
\begin{equation}
  \lambda_{\mathrm{MF}}(M,N,H,r)
  =\frac{r(M+N-r)}{2}+\ell(M_1,N_1,H_1).
  \label{eq_rank-removal}
\end{equation}
\end{lem}

\begin{proof}[{\bf Proof of Lemma~\ref{lem_rank_removal_identity}}]
This identity follows by direct substitution into Theorem~\ref{thm_aoyagi}.
Indeed, its three relevant branch conditions become
\begin{align}
\begin{aligned}
 M+r>N+H &\iff M_1>N_1+H_1,\\
 N+r>M+H &\iff N_1>M_1+H_1,\\
 M+r\leqq N+H\ \text{and}\ N+r\leqq M+H
   &\iff |M_1-N_1|\leqq H_1.
\end{aligned}    
\end{align}
Note that $H+r\leqq M+N$ is automatic from the assumption.
Since $(M+N+H+r)-(M_1+N_1+H_1)=4r$, the parity is also preserved.
Substitution into each of the three branches yields
Eq.~\eqref{eq_rank-removal}.
In the branch $M_1>N_1+H_1$,
\begin{align}
 \frac{r(M+N-r)}2+\ell(M_1,N_1,H_1)
 =\frac{r(M+N-r)+(H-r)(N-r)}2
 =\frac{HN-Hr+Mr}{2},
\end{align}
which is the corresponding expression in Theorem~\ref{thm_aoyagi}.
The other two branches follow by the same calculation.
\end{proof}

When $H=H_0>r$, $H+r \leqq M+N$ holds because $H=H_0\leq\min\{M,N\}$ and $r\leq\min\{M,N\}$.
Then, on account of Lemma~\ref{lem_local_reg} and Lemma~\ref{lem_rank_removal_identity},
the RLCT of NMF becomes
\begin{align}
\lambda &= \frac{MN-(M-a)(N-b)}{2}
   +\lambda_{\mathrm{MF}}(M-a,N-b,c,0) \\
&=\frac{r(M+N-r)}2
 +\frac{M_1t+N_1s-st}{2}
 +\ell(M_1-s,N_1-t,H_1-s-t).
\label{eq_nmf-reduced}
\end{align}
Consequently, let $\Delta = \lambda - \lambda_{\mathrm{MF}}$.
Then,
\begin{equation}
\Delta=\frac{M_1t+N_1s-st}{2}
 +\ell(M_1-s,N_1-t,H_1-s-t)-\ell(M_1,N_1,H_1).
\label{eq_delta}
\end{equation}
Here, we prove Proposition~\ref{prop_nmf_vs_mf_ranksmall}.

\begin{proof}[{\bf Proof of Proposition~\ref{prop_nmf_vs_mf_ranksmall}}]
Put the reduced triple
$I=M_1-s$, $J=N_1-t$, $K=H_1-s-t$.
First, we show that they are non-negative and satisfy $K \leqq \min\{I,J\}$.
Using $H=H_0$ and Sylvester's rank inequality, we have
\begin{align}
r=\mathrm{rank}(U_0V_0) &\geqq \mathrm{rank}(U_0)+\mathrm{rank}(V_0)-H \\
&=a+b-H.
\end{align}
Hence, we have $a+b \leqq H+r$.
Subtracting $2r$ from both sides,
\begin{align}
a+b-2r \leqq H-r=H_1.   
\end{align}
Since $a \geqq r$ and $b \geqq r$, $a+b-2r \geqq 0$, i.e.,$H_1-(a+b-2r) \geqq 0$.
By the definitions, $s+t=a+b-2r$. Then, we have
\begin{align}
K = H_1 -s-t \geqq 0.
\end{align}
Besides, $I-K=M_1-H_1+t \geqq 0$ and $J-K=N_1-H_1+s \geqq 0$ since $M \geqq H_0$, $N \geqq H_0$, $s \geqq 0$ and $t \geqq 0$.

Here, we evaluate \eqref{eq_delta} in all branches of Eq.~\eqref{eq_zero-rank-aw}.

\medskip
\noindent\textbf{Case 1: $M_1>N_1+H_1$, i.e., $M-N>H-r$.}
From the condition and $t \geqq 0$, 
\begin{align}
I-(J+K) &= M_1-N_1-H_1+2t \\
&\geqq M_1-N_1-H_1+t>0    
\end{align}
holds, both $\ell(M_1,N_1,H_1)$ and $\ell(I,J,K)$ are in the first branch of Eq.~\eqref{eq_zero-rank-aw}.
Substitution into Eq.~\eqref{eq_delta} gives
\begin{equation}
 \Delta
 =\frac{t(M_1-N_1-H_1+t)}{2}.
 \label{eq_delta-case1}
\end{equation}
Because of $M_1-N_1-H_1+t>0$ and $t \geqq 0$,
$\Delta$ is non-negative and $\Delta=0$ if and only if $t=0$, or equivalently $b=r$.

\medskip
\noindent\textbf{Case 2: $N_1>M_1+H_1$, i.e., $N-M>H-r$.}
In the same way as Case 1, we have
\begin{equation}
 \Delta
 =\frac{s(N_1-M_1-H_1+s)}{2}.
 \label{eq_delta-case2}
\end{equation}
$\Delta$ is non-negative and $\Delta=0$ if and only if $s=0$, or equivalently $a=r$.

\medskip
\noindent\textbf{Case 3: $|M_1-N_1|\leqq H_1$, i.e., $|M-N|\leqq H-r$.}
Set $\alpha=H_1+M_1-N_1$ and $\beta=H_1-M_1+N_1$.
From the conditions, $\alpha\geqq 0$ and $\beta \geqq 0$ hold.
Besides, $\alpha$ and $\beta$ satisfy $\alpha+\beta=2H_1$ and have the same parity.
Let
\begin{align}
\varepsilon=
 \begin{cases}
 0,&\alpha\text{ and }\beta\text{ are even},\\
 1,&\alpha\text{ and }\beta\text{ are odd}.
 \end{cases}
\end{align}
This is equal to $\delta_{M_1,N_1,H_1}$ in $\ell(M_1,N_1,H_1)$ because of
$M_1+N_1+H_1\equiv H_1+M_1-N_1 = \alpha \pmod 2$.

The two quantities $J-(I+K)$ and $I-(J+K)$ for the reduced triple satisfy
\begin{align}
 J-(I+K)&=2s-\alpha,\label{eq_reduced-imbalance-s}\\
 I-(J+K)&=2t-\beta.\label{eq_reduced-imbalance-t}
\end{align}
Moreover, since $s \geqq 0$, $t \geqq 0$ and $s+t \leqq H_1$, we have
\begin{equation}
 (2s-\alpha)+(2t-\beta)=2(s+t-H_1)\leqq 0,
 \label{eq_not-both-positive}
\end{equation}
so these two quantities in Eqs.~\eqref{eq_reduced-imbalance-s} and
\eqref{eq_reduced-imbalance-t} cannot both be positive simultaneously.
Then, the following three subcases are therefore exhaustive.

\medskip
\noindent\textbf{Case 3-1: $2s \leqq \alpha$ and $2t \leqq \beta$, i.e., $J \leqq I+K$ and $I \leqq J+K$.}
Because of the conditions and $K \geqq 0$, we have $|I-J| \leqq K$.
The parity term is still $\varepsilon$, owing to $I+J+K=M_1+N_1+H_1-2s-2t$.
Direct substitution into Eq.\eqref{eq_delta} gives
\begin{align}
\Delta &= \frac{M_1t+N_1s-st}{2}+\frac{Q(I,J,K)-Q(M_1,N_1,H_1)}{8}.\label{eq_delta-balbal-intermediate}
\end{align}
Developing the numerator of the second term, 
\begin{align}
&\quad Q(I,J,K)-Q(M_1,N_1,H_1) \nonumber \\
&= 2(H_1-s-t)(M_1+N_1-s-t)  -(M_1-N_1-s+t)^2-(H_1-s-t)^2 \nonumber \\
&\qquad - 2H_1(M_1+N_1)+(M_1-N_1)^2+H_1^2 \\
&= 2H_1(M_1+N_1)-2(s+t)(M_1+N_1)-2H_1(s+t)+2(s+t)^2 \nonumber \\
&\qquad +(2M_1-2N_1-s+t)(s-t)+(2H_1-s-t)(s+t) -2H_1(M_1+N_1) \\
&= -2(s+t)(M_1+N_1)-2H_1(s+t)+2(s+t)^2 \nonumber \\
&\qquad +2(M_1-N_1)(s-t)-(s-t)^2 + 2H_1(s+t) -(s+t)^2 \\
&= -4M_1t-4N_1s + (s+t)^2-(s-t)^2 \\
&= -4(M_1t+N_1s-st). \label{eq_qq-balbal}
\end{align}
Applying Eq.~\eqref{eq_qq-balbal} to Eq.~\eqref{eq_delta-balbal-intermediate}, we have
\begin{align}
\Delta = \frac{4(M_1t+N_1s-st)-4(M_1t+N_1s-st)}{8}=0.
 \label{eq_delta-balbal}
\end{align}

\medskip
\noindent\textbf{Case 3-2: $2s > \alpha$.}
We have $J>I+K$ by using Eq.~\eqref{eq_reduced-imbalance-s}; thus, the second branch of
Eq.~\eqref{eq_zero-rank-aw} applies to the reduced triple.  Direct expansion gives
\begin{align}
8\Delta &= 4(M_1t+N_1s-st)+4IK-Q(M_1,N_1,H_1)-\varepsilon \\
&= 4M_1t+4N_1s-4st+4(M_1-s)(H_1-s-t)  - 2H_1(M_1+N_1)+(M_1-N_1)^2+H_1^2 -\varepsilon  \\
&=  4N_1s + 2M_1H_1 -4M_1s -4H_1s +4s^2 -2N_1H_1 +(M_1-N_1)^2+H_1^2 -\varepsilon \\
&=  4s^2 -4(H_1+M_1-N_1)s + (M_1-N_1)^2+2H_1(M_1-N_1)+H_1^2 - \varepsilon \\
&= 4s^2 -4\alpha s + \alpha^2 - \varepsilon \\
&= (2s - \alpha)^2 - \varepsilon.
\end{align}
Hence, we have
\begin{equation}
 \Delta=\frac{(2s-\alpha)^2-\varepsilon}{8}.
 \label{eq_delta-balanced-s}
\end{equation}
If $\varepsilon=0$, the positive integer $2s-\alpha$ is even and hence at least $2$, i.e., $\Delta>0$.
If
$\varepsilon=1$, it is a positive odd integer, and
Eq.~\eqref{eq_delta-balanced-s} becomes 0 if and only if $2s-\alpha=1$, i.e.,
\begin{align}
s=\frac{\alpha+1}{2}=\left\lceil\frac{\alpha}{2}\right\rceil.    
\end{align}

\medskip
\noindent\textbf{Case 3-3: $2t > \beta$.}
A calculation symmetric to that in Case 3-2 gives
\begin{equation}
 \Delta=\frac{(2t-\beta)^2-\varepsilon}{8},
 \label{eq_delta-balanced-t}
\end{equation}
which vanishes in this subcase if and only if $\varepsilon=1$ and
\begin{align}
t=\frac{\beta+1}{2}=\left\lceil\frac{\beta}{2}\right\rceil.    
\end{align}

Combining the three subcases in Case 3, the equality holds exactly when
\begin{equation}
 s\leqq\left\lceil\frac{\alpha}{2}\right\rceil,
 \qquad
 t\leqq\left\lceil\frac{\beta}{2}\right\rceil.
 \label{eq_balanced-pq-condition}
\end{equation}
Indeed, when $\varepsilon=0$, these are precisely the balanced-subcase conditions.
When $\varepsilon=1$, feasibility $s+t\leqq H_1$ prevents both
$s=(\alpha+1)/2$ and $t=(\beta+1)/2$, since their sum would be $H_1+1$.
Thus, Eq.~\eqref{eq_balanced-pq-condition} consists of the balanced subcase plus the two one-step boundary cases in which
Eqs. \eqref{eq_delta-balanced-s} or \eqref{eq_delta-balanced-t} vanishes.
Finally, substituting
$s=a-r$, $t=b-r$, $\alpha=H-r+M-N$ and $\beta=H-r-M+N$
into Eq.~\eqref{eq_balanced-pq-condition} gives
Eq.~\eqref{eq_balanced-equality}.

The consideration of the above three cases also proves $\lambda \geqq \lambda_{\mathrm{MF}}$ and its strictness in every case.
\end{proof}

\section{\label{sec_conc}Conclusion}

In this study, we established a tighter upper bound on the RLCT of NMF than that obtained in previous work.
Furthermore, when the model and the true distribution are equal, we obtained the exact value of the RLCT over a broader range than was previously known.
The future directions are to determine the exact values when the model is redundant relative to the true distribution and to conduct large-scale numerical experiments.

\section*{Declaration of generative AI and AI-assisted technologies in the manuscript preparation process}
During the preparation of this work, the authors used OpenAI's ChatGPT as an auxiliary tool during the development of some proofs in this work, primarily for exploring possible lines of argument.
After using this tool,
the authors reviewed and edited the content as needed and take full responsibility for the content of the published article.

\bibliographystyle{plain}

\begin{thebibliography}{10}
\bibitem{Aoyagi1}
Miki Aoyagi and Sumio Watanabe.
\newblock Stochastic complexities of reduced rank regression in bayesian
  estimation.
\newblock {\em Neural Networks}, 18(7):924--933, 2005.

\bibitem{Bobadilla2018recommender}
Jes{\'u}s Bobadilla, Rodolfo Bojorque, Antonio~Hernando Esteban, and Remigio
  Hurtado.
\newblock Recommender systems clustering using bayesian non negative matrix
  factorization.
\newblock {\em IEEE Access}, 6:3549--3564, 2018.

\bibitem{Cemgil}
Ali~T. Cemgil.
\newblock Bayesian inference in non-negative matrix factorisation models.
\newblock {\em Computational Intelligence and Neuroscience}, 2009(4):17, 2009.
\newblock Article ID 785152.

\bibitem{Cohen}
Joel~E. Cohen and Uriel~G. Rothblum.
\newblock Nonnegative ranks, decompositions, and factorizations of nonnegative
  matrices.
\newblock {\em Linear Algebra and Its Applications}, 190:149--168, 1993.

\bibitem{Drton}
Mathias Drton and Martyn Plummer.
\newblock A bayesian information criterion for singular models.
\newblock {\em Journal of the Royal Statistical Society Series B}, 79:323--380,
  2017.
\newblock with discussion.

\bibitem{Finesso}
Lorenzo Finesso and Peter Spreij.
\newblock Nonnegative matrix factorization and i-divergence alternating
  minimization.
\newblock {\em Linear Algebra and its Applications}, 416(2-3):270--287, 2006.

\bibitem{nhayashi8}
Naoki Hayashi.
\newblock Variational approximation error in non-negative matrix factorization.
\newblock {\em Neural Networks}, 126:65--75, 2020.

\bibitem{nhayashi9}
Naoki Hayashi.
\newblock The exact asymptotic form of bayesian generalization error in latent
  dirichlet allocation.
\newblock {\em Neural Networks}, 137:127--137, 2021.

\bibitem{nhayashi5}
Naoki Hayashi and Sumio Watanabe.
\newblock Tighter upper bound of real log canonical threshold of non-negative
  matrix factorization and its application to bayesian inference.
\newblock In {\em IEEE Symposium Series on Computational Intelligence (IEEE
  SSCI)}, pages 718--725, 11 2017.

\bibitem{nhayashi2}
Naoki Hayashi and Sumio Watanabe.
\newblock Upper bound of bayesian generalization error in non-negative matrix
  factorization.
\newblock {\em Neurocomputing}, 266C(29 November):21--28, 2017.

\bibitem{Itakura}
F.~Itakura and S.~Saito.
\newblock Analysis synthesis telephony based on the maximum likelihood method.
\newblock In {\em Proc. 6th of the International Congress on Acoustics},
  1968.

\bibitem{Kim}
Hyunsoo Kim and Haesun Park.
\newblock Sparse non-negative matrix factorizations via alternating
  non-negativity-constrained least squares for microarray data analysis.
\newblock {\em Bioinformatics}, 23(12):1495--1502, 2007.
\newblock doi:10.1093/bioinformatics/btm134. PMID 17483501.

\bibitem{Kohjima}
Masahiro Kohjima, Tatsushi Matsubayashi, and Hiroshi Sawada.
\newblock Probabilistic non-negative inconsistent-resolution matrices
  factorization.
\newblock In {\em Proceeding of CIKM '15 Proceedings of the 24th ACM
  International on Conference on Information and Knowledge Management},
  volume~1, pages 1855--1858, 2015.

\bibitem{Kohjima2017phase}
Masahiro Kohjima and Sumio Watanabe.
\newblock Phase transition structure of variational bayesian nonnegative matrix
  factorization.
\newblock In {\em International Conference on Artificial Neural Networks},
  pages 146--154. Springer, 2017.

\bibitem{Lee}
Daniel~D. Lee and H.~Sebastian Seung.
\newblock Learning the parts of objects with nonnegative matrix factorization.
\newblock {\em Nature}, 401:788--791, 1999.

\bibitem{StatRethinkMcElreath2nd}
Richard McElreath.
\newblock {\em Statistical Rethinking: A Bayesian Course with Examples in R and
  Stan}.
\newblock CRC Press, 2nd edition, 2020.

\bibitem{Nagata2008asymptotic}
Kenji Nagata and Sumio Watanabe.
\newblock Asymptotic behavior of exchange ratio in exchange monte carlo method.
\newblock {\em Neural Networks}, 21(7):980--988, 2008.

\bibitem{Paatero}
Pentti Paatero and Unto Tapper.
\newblock Positive matrix factorization: A non-negative factor model with
  optimal utilization of error estimates of data values.
\newblock {\em Environmetrics}, 5(2):111--126, 1994.
\newblock doi:10.1002/env.3170050203.

\bibitem{Takeuchi2013NM2F}
Koh Takeuchi, Katsuhiko Ishiguro, Akisato Kimura, and Hiroshi Sawada.
\newblock Non-negative multiple matrix factorization.
\newblock In {\em IJCAI}, volume~13, pages 1713--1720, 2013.

\bibitem{Virtanen2008bayesian}
Tuomas Virtanen, A~Taylan Cemgil, and Simon Godsill.
\newblock Bayesian extensions to non-negative matrix factorisation for audio
  signal modelling.
\newblock In {\em Acoustics, Speech and Signal Processing, 2008. ICASSP 2008.
  IEEE International Conference on}, pages 1825--1828. IEEE, 2008.

\bibitem{Watanabe2}
Sumio Watanabe.
\newblock Algebraic geometrical methods for hierarchical learning machines.
\newblock {\em Neural Networks}, 13(4):1049--1060, 2001.

\bibitem{Watanabe2007almost}
Sumio Watanabe.
\newblock Almost all learning machines are singular.
\newblock In {\em 2007 IEEE Symposium on Foundations of Computational
  Intelligence}, pages 383--388. IEEE, 2007.

\bibitem{SWatanabeBookE}
Sumio Watanabe.
\newblock {\em Algebraic Geometry and Statistical Learning Theory}.
\newblock Cambridge University Press, 2009.

\bibitem{SWatanabeBookMath}
Sumio Watanabe.
\newblock {\em Mathematical theory of Bayesian statistics}.
\newblock CRC Press, 2018.

\bibitem{Xu}
W.~Xu, X.~Liu, and Y.~Gong.
\newblock Document clustering based on non-negative matrix factorization.
\newblock In {\em Proceedings of the 26th annual international ACM SIGIR
  conference on Research and development in information retrieval. Association
  for Computing Machinery}, pages 267--273, 2003.

\bibitem{Yamazaki2013comparing}
Keisuke Yamazaki and Daisuke Kaji.
\newblock Comparing two bayes methods based on the free energy functions in
  bernoulli mixtures.
\newblock {\em Neural Networks}, 44:36--43, 2013.

\bibitem{Yamazaki1}
Keisuke Yamazaki and Sumio Watanabe.
\newblock Singularities in mixture models and upper bounds of stochastic
  complexity.
\newblock {\em Neural Networks}, 16(7):1029--1038, 2003.

\bibitem{Zwiernik2011asymptotic}
Piotr Zwiernik.
\newblock An asymptotic behaviour of the marginal likelihood for general markov
  models.
\newblock {\em Journal of Machine Learning Research}, 12(Nov):3283--3310, 2011.
\end{thebibliography}

\end{document}